\documentclass{article}

\usepackage{PRIMEarxiv}
\usepackage[utf8]{inputenc}
\usepackage[T1]{fontenc}
\usepackage{hyperref}
\usepackage{url}
\usepackage{booktabs}
\usepackage{nicefrac}
\usepackage{microtype}
\usepackage{lipsum}
\usepackage{fancyhdr}
\usepackage{graphicx}
\usepackage{subcaption}
\graphicspath{{pic/}}
\usepackage{amsmath,amssymb,amsthm,amsfonts}
\usepackage{mathtools}
\usepackage{thmtools}
\usepackage{natbib}
\usepackage{algorithm}
\usepackage{algorithmic}
\usepackage{xcolor}
\usepackage{enumitem}
\usepackage{bbm}
\usepackage{dsfont}
\usepackage{pifont}
\usepackage{bm}

\newcommand{\R}{\mathbb{R}}
\newcommand{\E}{\mathbb{E}}
\newcommand{\Hk}{\mathcal{H}_k}
\newcommand{\Tk}{\mathcal{T}_k}
\newcommand{\KL}{\mathrm{KL}}
\newcommand{\MMD}{\mathrm{MMD}}
\newcommand{\bx}{\mathbf{x}}
\newcommand{\by}{\mathbf{y}}
\newcommand{\beps}{\boldsymbol{\epsilon}}
\newcommand{\inner}[2]{\langle #1, #2 \rangle}
\newcommand{\norm}[1]{\| #1 \|}
\newcommand{\Sph}{\mathbb{S}}
\DeclareMathOperator{\Proj}{Proj}

\newtheorem{theorem}{Theorem}[section]
\newtheorem{proposition}[theorem]{Proposition}

\newtheorem{corollary}[theorem]{Corollary}
\newtheorem{definition}[theorem]{Definition}
\newtheorem{assumption}[theorem]{Assumption}
\newtheorem{remark}[theorem]{Remark}

\title{Gradient Flow Drifting: \\
Generative Modeling via Wasserstein Gradient Flows of KDE-Approximated Divergences}

\author{
  Jiarui Cao, Zixuan Wei \\
  The Chinese University of Hong Kong \\
  Hong Kong\\
  \texttt{\{1155244613, 1155245852\}@link.cuhk.edu.hk} \\
   \And
  Yuxin Liu \\
  Civil Aviation University of China \\
  Tianjin\\
  \texttt{yx\_liu2025061016@163.com} \\
}

\date{}

\begin{document}

\maketitle

\begin{abstract}
We reveal a precise mathematical framework about a new family of generative models which we call \textbf{Gradient Flow Drifting}. With this framework, we prove an equivalence between the recently proposed Drifting Model and the Wasserstein gradient flow of the forward KL divergence under kernel density estimation (KDE) approximation. Specifically, we prove that the drifting field of drifting model~\cite{deng2026generative} equals, up to a bandwidth-squared scaling factor, the difference of KDE log-density gradients $\nabla \log p_{\mathrm{kde}} - \nabla \log q_{\mathrm{kde}}$, which is exactly the particle velocity field of the Wasserstein-2 gradient flow of $\KL(q\|p)$ with KDE-approximated densities. Besides that, this broad family of generative models can also include MMD-based generators, which arises as special cases of Wasserstein gradient flows of different divergences under KDE approximation. We provide a concise identifiability proof, and a theoretically grounded mixed-divergence strategy. We combine reverse KL and $\chi^2$ divergence gradient flows to simultaneously avoid mode collapse and mode blurring, and extend this method onto Riemannian manifold which loosens the constraints on the kernel function, and makes this method more suitable for the semantic space. Preliminary experiments on synthetic benchmarks validate the framework.
\end{abstract}

\section{Introduction}
\label{sec:intro}

Generative modeling seeks to learn a mapping $f$ such that the pushforward $f_\# p_{\beps}$ of a simple prior $p_{\beps}$ approximates a data distribution $p_{\mathrm{data}}$. The recently proposed \emph{Drifting Model}~\cite{deng2026generative} introduces a new paradigm: rather than relying on iterative inference-time dynamics (as in diffusion or flow-based models), it evolves the pushforward distribution during \emph{training time} via a \emph{drifting field} $\mathbf{V}_{p,q}$, and naturally admits one-step generation. Drifting Models achieve state-of-the-art one-step FID on ImageNet $256\times 256$ (1.54 in latent space and 1.61 in pixel space).

Despite their empirical success, theoretical foundations of Drifting Models remain underdeveloped. The original paper's analysis is somewhat heuristic and the identifiability proof (Appendix C.1 therein) requires additional smoothness assumptions. We argue that this complexity stems from a failure to recognize a fundamental connection.

\paragraph{Our key observation.} The drifting field of~\cite{deng2026generative}, when instantiated with a Gaussian kernel $k_h(\bx,\by) = \exp(-\frac{\|\bx-\by\|^2}{2h^2})$, satisfies the \emph{exact} identity:
\begin{equation}
\label{eq:main-identity}
\mathbf{V}_{p,q}(\bx) = h^2 \bigl(\nabla \log p_{\mathrm{kde}}(\bx) - \nabla \log q_{\mathrm{kde}}(\bx)\bigr),
\end{equation}
where $p_{\mathrm{kde}}(\bx) = \E_p[k_h(\bx,\by)]$ is the Kernel Density Estimation (KDE) of $p$. The right-hand side is precisely the particle velocity field of the \textbf{Wasserstein-2 gradient flow} of the \textbf{KL divergence} $\KL(q\|p)$, with true densities replaced by their KDE approximations with the same kernel.

This identification has several consequences:
\begin{enumerate}[leftmargin=*,nosep]
    \item \textbf{Unified framework}: By varying the divergence functional, we obtain a family of gradient flow drifting models. MMD-based generators correspond to the $\mathcal{L}^2$ distribution distance, and drifting models to the KL divergence. We can construct new models from any $f$-divergence and any other divergence that can prove the distribution convergence.
    \item \textbf{Mixed gradient flows}: Convex combinations of divergences yield legitimate mixed gradient flows (Theorem~\ref{thm:mixed-flow}), enabling strategies that combine the complementary strengths of different divergences—e.g., MMD for global mode coverage and reverse KL for local sharpness, and reverse KL divergence and $\chi^2$ provide more specific precise forcing.
    \item \textbf{Simplified identifiability}: The equilibrium condition $\mathbf{V}_{p,q}=\mathbf{0} \Rightarrow p=q$ follows in lines from the injectivity of the kernel mean embedding under characteristic kernels.
    \item \textbf{Drifting Model as a special case}: The standard energy dissipation inequality for Wasserstein gradient flows immediately yields $\frac{d}{dt}\KL(q_t^{\mathrm{kde}}\|p^{\mathrm{kde}}) \leq 0$.
\end{enumerate}

Concurrent work by~\cite{li2026long} reinterprets Drifting Models through a flow-map semigroup decomposition, but does not identify the KDE--gradient flow connection.~\cite{belhadji2025weighted} unifies MMD gradient flows with mean shift but does not extend to $f$-divergences. Our framework subsumes both perspectives.

\section{Related Work}
\label{sec:related}

\paragraph{Drifting Models.}
~\cite{deng2026generative} propose learning a one-step pushforward map by evolving the generated distribution during training via a kernel-based drifting field. They achieve strong empirical results but provide limited theoretical analysis.~\cite{li2026long} reinterpret Drifting Models via long-short flow-map factorization, connecting them to closed-form flow matching based on semigroup consistency. Our work gives a new perspective to view it and includes it into a big family of generative models.

\paragraph{Wasserstein gradient flows in generative modeling.}
Wasserstein gradient flows~\cite{jordan1998variational,ambrosio2005gradient,Santambrogio2015} provide a variational framework for the evolution of probability measures. Several works leverage this framework for generative modeling:~\cite{arbel2019maximum} study MMD gradient flows for sampling;~\cite{yi2023monoflow} use Wasserstein gradient flows to unify divergence GANs, introducing MonoFlow with a monotone rescaling of the log density ratio;~\cite{choi2024scalable} propose scalable Wasserstein gradient descent. Our work makes optimizing the gradient work directly possible through kernel density estimation.

\paragraph{Kernel density estimation and score estimation.}
The connection between mean shift and KDE gradients is classical~\cite{cheng1995mean,comaniciu2002mean}.~\cite{belhadji2025weighted} recently unified mean shift, MMD-optimal quantization, and gradient flows. Our work extends this connection to arbitrary $f$-divergences.

\paragraph{MMD and kernel methods for generation.}
MMD-based generative models~\cite{dziugaite2015training,li2015generative} minimize the MMD between generated and data distributions.~\cite{zhou2025inductive} extend moment matching to one-/few-step diffusion.~\cite{chizat2026quantitative} provide quantitative convergence rates for MMD Wasserstein gradient flows. Our framework reveals MMD generators as one member of a broader family.

\paragraph{$f$-divergence minimization.}
$f$-divergence variational estimation has been widely studied~\cite{nguyen2010estimating,nowozin2016f}.~\cite{yi2023monoflow} connects $f$-divergence GANs to Wasserstein gradient flows but requires a discriminator to estimate density ratios. Our KDE-based approach avoids adversarial training entirely, but aligns with an approximation based on particulars.

\section{Preliminaries}
\label{sec:prelim}

\subsection{Notation}

Let $\mathcal{P}(\R^d)$ denote the set of Borel probability measures on $\R^d$, and $\mathcal{P}_2(\R^d)$ the subset with finite second moments. For $\mu \in \mathcal{P}(\R^d)$, we write $\mu$ also for its density with respect to Lebesgue measure when it exists. We use $\inner{\cdot}{\cdot}$ for inner products and $\norm{\cdot}$ for norms, with subscripts indicating the space when ambiguous.

\subsection{Kernel Density Estimation}

\begin{definition}[KDE operator]
\label{def:kde}
Given a kernel $k: \R^d \times \R^d \to \R$ and $\mu \in \mathcal{P}(\R^d)$, the \emph{KDE operator} is
\begin{equation}
\Tk[\mu](\bx) := \int_{\R^d} k(\bx, \by) \mathrm{d}\mu(\by).
\end{equation}
For the Gaussian kernel $k_h(\bx,\by) = \exp(-\|\bx-\by\|^2/(2h^2))$ with bandwidth $h > 0$, we write $\mu_{\mathrm{kde}}(\bx) := \mathcal{T}_{k_h}[\mu](\bx)$.
\end{definition}

\subsection{Reproducing Kernel Hilbert Spaces}

\begin{definition}[RKHS and kernel mean embedding]
\label{def:rkhs}
A symmetric positive definite kernel $k$ induces a unique reproducing kernel Hilbert space $\Hk$ with inner product $\inner{\cdot}{\cdot}_{\Hk}$ satisfying the reproducing property: $f(\bx) = \inner{f}{k(\bx,\cdot)}_{\Hk}$ for all $f \in \Hk$. The \emph{kernel mean embedding} of $\mu \in \mathcal{P}(\R^d)$ is $m_k^\mu := \int k(\cdot,\by)\mathrm{d}\mu(\by) \in \Hk$.
\end{definition}

\begin{definition}[Characteristic kernel]
\label{def:characteristic}
A kernel $k$ is \emph{characteristic} if the kernel mean embedding map $\mu \mapsto m_k^\mu$ is injective on $\mathcal{P}(\R^d)$.
\end{definition}

\subsection{Wasserstein Gradient Flows}

\begin{definition}[Wasserstein-2 gradient flow]
\label{def:wgf}
Given a functional $\mathcal{F}: \mathcal{P}_2(\R^d) \to \R$ with first variation $\frac{\delta \mathcal{F}}{\delta q}$, its \emph{Wasserstein-2 gradient flow} is the curve $\{q_t\}_{t \geq 0}$ satisfying the continuity equation:
\begin{equation}
\label{eq:continuity}
\partial_t q_t = \nabla \cdot \left(q_t \nabla \frac{\delta \mathcal{F}}{\delta q}\bigg|_{q_t}\right).
\end{equation}
Equivalently, particles $\bx_t \sim q_t$ evolve as $\frac{d\bx_t}{dt} = \mathbf{v}(\bx_t)$ where $\mathbf{v}(\bx) = -\nabla \frac{\delta \mathcal{F}}{\delta q}(\bx)$.
\end{definition}

\subsection{The Drifting Model}

We recall the core formulation of~\cite{deng2026generative}. Given a data distribution $p$ and a generated distribution $q = f_\#p_{\beps}$, the drifting field is:
\begin{equation}
\label{eq:drifting-field-original}
\mathbf{V}_{p,q}(\bx) = \underbrace{\frac{\E_p[k(\bx,\by^+)(\by^+ - \bx)]}{\E_p[k(\bx,\by^+)]}}_{\mathbf{V}_p^+(\bx)} - \underbrace{\frac{\E_q[k(\bx,\by^-)(\by^- - \bx)]}{\E_q[k(\bx,\by^-)]}}_{\mathbf{V}_q^-(\bx)},
\end{equation}
with training loss $\mathcal{L} = \E_{\beps}[\|f_\theta(\beps) - \mathrm{stopgrad}(f_\theta(\beps) + \mathbf{V}_{p,q_\theta}(f_\theta(\beps)))\|^2]$.

\section{Method: Gradient Flow Drifting}
\label{sec:method}

We present a unified framework in which generative models arise as Wasserstein gradient flows (WGFs) of divergence functionals under KDE approximation. The logical development proceeds in three layers:
\begin{itemize}[nosep,leftmargin=*]
    \item \textbf{Foundation} \ref{sec:kde-matching}: Under mild kernel regularity conditions, KDE-level distribution matching is equivalent to matching the original distributions.
    \item \textbf{Engine} \ref{sec:f-div-flows}--\ref{sec:ident-conv}: General $f$-divergence WGFs at the KDE level, with energy dissipation and unified identifiability.
    \item \textbf{Instantiation} \ref{sec:drifting-case}--\ref{sec:mixed}: The Drifting Model, MMD generators, and mixed gradient flows emerge as special cases.
\end{itemize}
The framework extends naturally to Riemannian manifolds \ref{sec:riemannian}, and is summarized as a complete training pipeline in \ref{sec:pipeline}.

\subsection{Foundation: KDE Smoothing and Distribution Matching}
\label{sec:kde-matching}

The starting point of our framework is that KDE smoothing, under mild kernel regularity, preserves distributional identity and simultaneously provides the smoothness needed for gradient flow analysis. This allows us to work entirely at the KDE level without imposing any regularity on the data distribution $p$ or the generated distribution $q$.

\begin{assumption}[Kernel regularity; full statement in Appendix~\ref{app:assumptions}]
\label{ass:kernel-main}
Let $k:\R^d\times\R^d\to\R$ satisfy:
\begin{enumerate}[nosep, label=\textbf{K\arabic*.}]
    \item \textbf{Characteristic}: the mean embedding $\mu \mapsto \int k(\cdot,\by)\mathrm{d}\mu(\by)$ is injective on $\mathcal{P}(\R^d)$.
    \item \textbf{Uniform gradient bound}: $M_k := \sup_{\bx,\by}\|\nabla_\bx k(\bx,\by)\| < \infty$.
    \item \textbf{Strict positivity}: $k(\bx,\by) > 0$ for all $\bx,\by$.
    \item \textbf{Differentiability}: $\bx \mapsto k(\bx,\by)$ is $C^1$ for every $\by$.
\end{enumerate}
\end{assumption}

The Gaussian kernel $k_h(\bx,\by)=\exp(-\|\bx-\by\|^2/(2h^2))$ satisfies \ref{K1}--\ref{K4}; the Laplace kernel used in the original Drifting Model fails \ref{K4} (Appendix~\ref{app:kernels}).

\begin{theorem}[KDE regularity; proof in Appendix~\ref{app:regularity}]
\label{thm:kde-regularity}
Under \textbf{K2}--\textbf{K4}, for any $\mu \in \mathcal{P}(\R^d)$:
(i)~$\mu_\mathrm{kde} \in C^1(\R^d)$ with $\nabla_\bx \mu_\mathrm{kde}(\bx) = \int \nabla_\bx k(\bx,\by)\mathrm{d}\mu(\by)$;
(ii)~$\mu_\mathrm{kde}(\bx) > 0$ for all $\bx$;
(iii)~$\sup_\bx \|\nabla \mu_\mathrm{kde}(\bx)\| \leq M_k$.
\end{theorem}

In particular, no moment or smoothness conditions on $\mu$ are required: the constant $M_k$ serves as a universal dominating function for any probability measure, enabling all subsequent Leibniz interchanges.

\begin{proposition}[KDE injectivity; proof in Appendix~\ref{app:injectivity}]
\label{prop:kde-injectivity}
Under \textbf{K1}, $\mu_\mathrm{kde} = \nu_\mathrm{kde}$ pointwise implies $\mu = \nu$.
\end{proposition}

\begin{remark}[Foundation summary] 
\label{rem:foundation-summary}
Under \ref{K1}--\ref{K4}, the KDE-smoothed densities $p_\mathrm{kde}$ and $q_\mathrm{kde}$ are strictly positive and $C^1$. In particular, the log-ratio $\log(p_\mathrm{kde}/q_\mathrm{kde})$ is well-defined and $C^1$, and $p_\mathrm{kde} = q_\mathrm{kde}$ if and only if $p = q$. This means every divergence-minimization argument at the KDE level faithfully transfers to the original distributions.
\end{remark}

\subsection{Gradient Flows of $f$-Divergences under KDE Approximation}
\label{sec:f-div-flows}

With smoothed densities that are smooth and positive (Theorem~\ref{thm:kde-regularity}), we can apply the standard Wasserstein gradient flow machinery to $f$-divergences directly at the KDE level.

Recall that for a convex function $f:(0,\infty)\to\R$ with $f(1)=0$, the $f$-divergence is $D_f(\rho\|\pi) = \int \pi\, f(\rho/\pi)\mathrm{d}\bx$ (Definition~\ref{def:f-div} in Appendix). The WGF of $\mathcal{F}[q] = D_f(q\|p)$ has first variation $\frac{\delta\mathcal{F}}{\delta q}(\bx) = f'(q(\bx)/p(\bx))$ and particle velocity $\mathbf{v}_f(\bx) = -\nabla f'(q(\bx)/p(\bx))$ (Proposition~\ref{prop:wgf} in Appendix). Replacing the true densities with their KDE approximations yields the \emph{generalized drifting velocity field}:
\begin{equation}
\label{eq:generalized-velocity}
\mathbf{v}_f(\bx) = -\nabla f'\!\left(\frac{q(\bx)}{p(\bx)}\right).
\end{equation}

\begin{theorem}[Energy dissipation]
\label{thm:energy-dissipation}
Let $f$ be strictly convex and let $\{q_t\}_{t \geq 0}$ be smooth positive densities evolving according to the continuity equation with velocity~\eqref{eq:generalized-velocity}. Under appropriate boundary conditions (Appendix~\ref{app:f-div}, Remark~\ref{rem:boundary}):
\begin{equation}
\label{eq:energy-dissipation}
\frac{d}{dt}D_f(q_{t}\|p) = -\int q_{t}(\bx)\left\|\nabla f'\!\left(\frac{q_{t}(\bx)}{p(\bx)}\right)\right\|^2 \mathrm{d}\bx \leq 0.
\end{equation}
On compact Riemannian manifolds without boundary (e.g., $\Sph^{d-1}$), the boundary condition is vacuous and~\eqref{eq:energy-dissipation} holds unconditionally.
\end{theorem}

Table~\ref{tab:unification} records the specific velocity fields for the divergences of primary interest.

\begin{table}[t]
\centering
\caption{Generative models as Wasserstein gradient flows of divergences under KDE approximation. All velocity fields are sample-computable via the KDE score formula (Appendix~\ref{app:score}).}
\label{tab:unification}
\begin{tabular}{@{}lllll@{}}
\toprule
$f(u)$ & Divergence & $f'(u)$ & KDE velocity field $\mathbf{v}^{\mathrm{kde}}(\bx)$ & Model \\
\midrule
$u\log u$ & Forward KL & $\log u + 1$ & $\nabla \log p_\mathrm{kde} - \nabla \log q_\mathrm{kde}$ & Drifting \\
$-\log u$ & Reverse KL & $-1/u$ & $\frac{p_\mathrm{kde}}{q_\mathrm{kde}}(\nabla \log p_\mathrm{kde} - \nabla \log q_\mathrm{kde})$ & -- \\[2pt]
$\frac{1}{2}(u-1)^2$ & $\chi^2$ & $(u-1)$ & $\frac{q_\mathrm{kde}}{p_\mathrm{kde}}(\nabla \log p_\mathrm{kde} - \nabla \log q_\mathrm{kde})$ & -- \\
\midrule
\multicolumn{2}{l}{$\frac{1}{2}\|m_k^p - m_k^q\|_{\Hk}^2$} & \multicolumn{2}{l}{$\nabla_\bx \int k(\bx,\by)\mathrm{d}(p-q)(\by) = \nabla(p_\mathrm{kde} - q_\mathrm{kde})$} & MMD \\
\bottomrule
\end{tabular}
\end{table}

\begin{remark}[Factored velocity structure]
\label{rem:factored}
All $f$-divergence velocities in Table~\ref{tab:unification} share the common factor $(\nabla\log p_\mathrm{kde} - \nabla\log q_\mathrm{kde})$, modulated by a density-ratio weight $w(\bx)$: $w \equiv 1$ (forward KL), $w = p_\mathrm{kde}/q_\mathrm{kde}$ (reverse KL), $w = q_\mathrm{kde}/p_\mathrm{kde}$ ($\chi^2$). This weight governs the local emphasis: forward KL treats all regions equally, reverse KL up-weights regions of high data density (encouraging precision), and $\chi^2$ up-weights regions of high generated density (penalizing spurious mass).
\end{remark}

\subsection{Unified Identifiability}
\label{sec:ident-conv}

Combining the distribution-matching foundation \ref{sec:kde-matching} with the gradient flow machinery \ref{sec:f-div-flows}, we obtain a unified identifiability result.

\begin{theorem}[Unified identifiability (Proof in Appendix \ref{app:proof-ident})]
\label{thm:identifiability}
Let $k$ satisfy \ref{K1}--\ref{K4} and $f$ be strictly convex with $f(1)=0$. If the generalized drifting velocity~\eqref{eq:generalized-velocity} vanishes identically, $\mathbf{v}_f^{\mathrm{kde}} \equiv \mathbf{0}$, then $p = q$.
\end{theorem}

\begin{corollary}[Loss landscape]
\label{cor:landscape}
The KDE-level $f$-divergence $D_f(q_\mathrm{kde}\|p_\mathrm{kde})$
satisfies:
\begin{enumerate}[nosep]
    \item $D_f \geq 0$, with equality if and only if $q = p$
    \emph{(identifiability)};
    \item $\frac{d}{dt}D_f \leq 0$ along the Wasserstein gradient
    flow \emph{(energy dissipation)};
    \item the only equilibrium of the flow is $q = p$.
\end{enumerate}
\end{corollary}

The unique global optimum of the KDE-level divergence is $p=q$, and the energy is monotonically non-increasing along the flow.

\subsection{The Drifting Model as Forward KL Gradient Flow}
\label{sec:drifting-case}

We now show that the Drifting Model of~\cite{deng2026generative} is a special case of our framework, corresponding to the forward KL divergence $f(u) = u \log u$.

\begin{theorem}[Core equivalence; proof in Appendix~\ref{app:equivalence}]
\label{thm:core-equivalence}
Let $k_h(\bx,\by) = \exp(-\|\bx-\by\|^2/(2h^2))$ with $h > 0$, and let $p, q \in \mathcal{P}(\R^d)$. Then the drifting field~\eqref{eq:drifting-field-original} satisfies
\begin{equation}
\label{eq:core-eq}
\mathbf{V}_{p,q}(\bx) = h^2\bigl(\nabla \log p_\mathrm{kde}(\bx) - \nabla \log q_\mathrm{kde}(\bx)\bigr) \quad \text{for all } \bx \in \R^d.
\end{equation}
\end{theorem}

The proof is a direct computation: the Gaussian kernel satisfies $\nabla_\bx k_h(\bx,\by) = \frac{\by - \bx}{h^2}k_h(\bx,\by)$, and substituting into the KDE score formula (Appendix~\ref{app:score}) gives exactly the mean-shift vectors $\mathbf{V}_p^+$ and $\mathbf{V}_q^-$ from~\eqref{eq:drifting-field-original}.

\begin{corollary}[Drifting = Forward KL Wasserstein gradient flow + KDE]
\label{cor:drifting-wgf}
The right-hand side of~\eqref{eq:core-eq} is precisely $h^2 \mathbf{v}_{\KL}^{\mathrm{kde}}(\bx)$, the forward KL row of Table~\ref{tab:unification} scaled by $h^2$. Hence the Drifting Model's velocity fields correspond to the Wasserstein-2 gradient flow of $\KL(q_\mathrm{kde}\|p_\mathrm{kde})$, up to a time rescaling by $h^2$.
\end{corollary}

This identification immediately imports the convergence and identifiability results of \ref{sec:ident-conv} to the Drifting Model. The identifiability proof, in particular, reduces to:
$\mathbf{V}_{p,q} \equiv \mathbf{0}
\;\xRightarrow{\text{Thm.~\ref{thm:core-equivalence}}}\;
\nabla\log p_\mathrm{kde} = \nabla\log q_\mathrm{kde}
\;\xRightarrow{\text{Thm.~\ref{thm:identifiability}}}\;
p = q.$

\subsection{MMD Generators as $\mathcal{L}^2$ Gradient Flows}
\label{sec:mmd-case}

The squared MMD functional $\mathcal{F}[q] = \frac{1}{2}\MMD_k^2(q,p) = \frac{1}{2}\|m_k^q - m_k^p\|_{\Hk}^2$ is not an $f$-divergence, but fits naturally into our framework.

\begin{proposition}[MMD gradient flow velocity; proof in Appendix~\ref{app:mmd}]
\label{prop:mmd-velocity}
Under \ref{K1}--\ref{K4}, the WGF velocity of $\frac{1}{2}\MMD_k^2(q,p)$ is
\begin{equation}
\mathbf{v}_\MMD(\bx) = \nabla\bigl(p_\mathrm{kde}(\bx) - q_\mathrm{kde}(\bx)\bigr) = \int \nabla_\bx k(\bx,\by)\mathrm{d}(p-q)(\by).
\end{equation}
\end{proposition}

Note that $\mathbf{v}_\MMD$ is the gradient of the $\mathcal{L}^2$ density difference, while the $f$-divergence velocities in Table~\ref{tab:unification} involve the gradient of a nonlinear function of the density ratio. Both families are sample-computable via the KDE score formula, and the same identifiability argument applies (Remark in \ref{sec:ident-conv}).

\subsection{Mixed Gradient Flows}
\label{sec:mixed}

Different divergences induce complementary failure modes. We propose mixing gradient flows to combine their strengths.

\begin{theorem}[Legitimacy of mixed gradient flows; proof in Appendix~\ref{app:superposition}]
\label{thm:mixed-flow}
Let $D_1, D_2$ be divergences ($D_i(q\|p) \geq 0$, with equality iff $q=p$), and $\alpha, \beta > 0$ with $\alpha + \beta = 1$. Define $D_\mathrm{mix} = \alpha D_1 + \beta D_2$. Then:
\begin{enumerate}[nosep]
    \item[(a)] $D_\mathrm{mix}$ is a valid divergence;
    \item[(b)] its WGF velocity is $\mathbf{v}_\mathrm{mix} = \alpha\,\mathbf{v}_1 + \beta\,\mathbf{v}_2$;
    \item[(c)] $\frac{d}{dt}D_\mathrm{mix}[q_t] \leq 0$ along the flow.
\end{enumerate}
\end{theorem}

\paragraph{Practical mixed drifting field.} We propose combining the reverse KL and $\chi^2$ velocity fields:
\begin{equation}
\label{eq:mixed-field}
\mathbf{V}_\mathrm{mix}(\bx) = \alpha \cdot \frac{p_\mathrm{kde}}{q_\mathrm{kde}}(\nabla \log p_\mathrm{kde} - \nabla \log q_\mathrm{kde}) + \beta \cdot \frac{q_\mathrm{kde}}{p_\mathrm{kde}}(\nabla \log p_\mathrm{kde} - \nabla \log q_\mathrm{kde}),
\end{equation}
corresponding to $D_\mathrm{mix} = \alpha\,\KL_\mathrm{kde}(p\|q) + \beta\,\chi^2_\mathrm{kde}(q\|p)$. Referring to Remark~\ref{rem:factored}, the reverse KL weight $p_\mathrm{kde}/q_\mathrm{kde}$ provides strong attraction toward high-density regions of $p$ (precision-forcing, avoiding mode blurring), while the $\chi^2$ weight $q_\mathrm{kde}/p_\mathrm{kde}$ penalizes spurious generated mass (coverage-forcing, avoiding mode collapse). Their combination reconciles mode-seeking and mode-covering behaviors. Experiments in \ref{sec:experiments} confirm this qualitative picture.

\subsection{Extension to Riemannian Manifolds}
\label{sec:riemannian}

The Drifting Model~\cite{deng2026generative} trains in a semantic feature space that is empirically close to a hypersphere. This motivates extending our framework to Riemannian manifolds $M$. Two benefits emerge:
\begin{enumerate}[nosep,leftmargin=*]
    \item \textbf{Vacuous boundary conditions.} On compact manifolds without boundary (e.g., $\Sph^{d-1}$), the energy dissipation inequality~\eqref{eq:energy-dissipation} holds unconditionally (Theorem~\ref{thm:energy-dissipation}), eliminating the tail-decay assumptions required on $\R^d$.
    \item \textbf{Richer kernel design.} The adapted spherical assumptions \textbf{K1}$_\Sph$--\textbf{K4}$_\Sph$ (Appendix~\ref{app:kernels-sphere}) admit kernels with qualitatively different weighting profiles. For example, the von Mises--Fisher (vMF) kernel $k_\kappa(\bx,\by) = \exp(\kappa\,\bx^\top\by)$ provides a spherical analog of the Gaussian kernel, while the spherical logarithmic kernel (Appendix~\ref{app:kernels-sphere}, Proposition~\ref{prop:spherical-log}) produces polynomial (inverse-distance) weighting, analogous to the Euclidean IMQ kernel, offering heavier tails and better global mode coverage.
\end{enumerate}

All results of \ref{sec:f-div-flows}--\ref{sec:mixed}—velocity fields, energy dissipation, identifiability, and mixed flows—extend to the Riemannian setting by replacing Euclidean gradients with Riemannian gradients and requiring the manifold analogues of \ref{K1}--\ref{K4}. Details and kernel verifications are given in Appendix~\ref{app:kernels-sphere}.

\subsection{Training Pipeline}
\label{sec:pipeline}

Algorithm~\ref{alg:drift} summarizes the full training procedure of Gradient Flow Drifting. The framework is modular: one selects a divergence (or mixture), a kernel satisfying \ref{K1}--\ref{K4}, and trains a one-step generator via the stop-gradient loss.

\begin{algorithm}[t]
\caption{Gradient Flow Drifting: training algorithm}
\label{alg:drift}
\begin{algorithmic}[1]
\REQUIRE Generator $f_\theta$, data distribution $p$, source distribution $p_{\beps}$

\STATE \textbf{Divergence selection:} Choose divergence(s) achieving distributional convergence (e.g., reverse KL, forward KL, $\chi^2$, or mixture thereof).

\STATE \textbf{Velocity field:} Derive the WGF velocity from the chosen divergence.

\STATE \textbf{Kernel design:} Select a kernel $k$ satisfying Assumption~\ref{ass:kernel-main} \ref{K1}--\ref{K4}, or their Riemannian analogues.

\FOR{each training iteration}

\STATE Sample $\beps \sim p_{\beps}$; compute $\bx = f_\theta(\beps)$ (generated samples).

\STATE Sample $\by^+ \sim p$ (data samples).

\STATE \textbf{Mini-batch KDE velocity estimation:} Compute $\mathbf{v}^{\mathrm{kde}}(\bx)$ over $\{\by^+\}$ and $\{\bx\}$.

\STATE \textbf{Update:} $\mathcal{L}(\theta) = \E_{\beps}\bigl[\|f_\theta(\beps) - \mathrm{sg}(f_\theta(\beps) + \mathbf{v}^{\mathrm{kde}}(f_\theta(\beps)))\|^2\bigr]$; \; $\theta \leftarrow \theta - \eta\nabla_\theta\mathcal{L}$.

\ENDFOR
\end{algorithmic}
\end{algorithm}

\section{Experiments}
\label{sec:experiments}

\subsection{Synthetic 2D Benchmarks}
We visualized the particle evolution under the velocity field of gradient flow using different implementations of divergence and kernel functions. 

As shown in Fig.\ref{fig:swiss_images}, the original drifting model and $\mathcal{L}^2$ flow drifting (with the same gradient flow with MMD) show a mode-covering training process as harsh punishment from divergence. We can easily find that they both have blur situations.

The reverse KL divergence + $\chi^2$ divergence
mixture flow drifting shows a totally different evolving process. This model almost only generate precise samples, but not struggles in mode collapse, it quickly explored all the modes.

The original drifting model uses laplace kernel which may have some issues in high probability area. Since the Laplace kernel violates the assumption\ref{K4}, the gradient flow derived from it is mathematically only "weakly" defined, and during the convergence stage, it causes numerical instability (jittering) of particles near the data manifold. We can observe this phenomenon on the center of the swiss-roll distribution, the generation distribution has weird distortion, while RBF kernel version not. While the drifting model achieved empirical success due to the uniformity of semantic distribution, we can make it much more stable through the design of kernel function.

\begin{figure}[t]
\centering

\begin{subfigure}{0.8\textwidth}
    \centering
    \includegraphics[width=\textwidth]{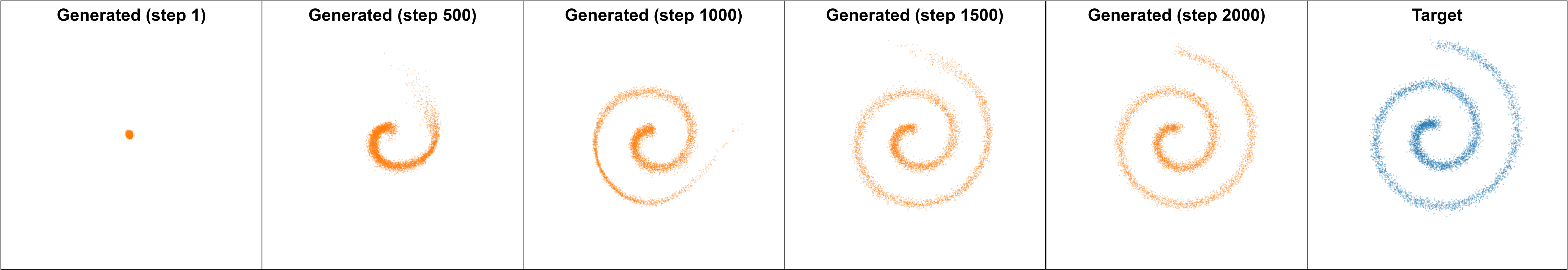}
    \caption{Reverse KL divergence + $\chi^2$ divergence (RBF kernel)}
\end{subfigure}

\vspace{0.4cm}

\begin{subfigure}{0.8\textwidth}
    \centering
    \includegraphics[width=\textwidth]{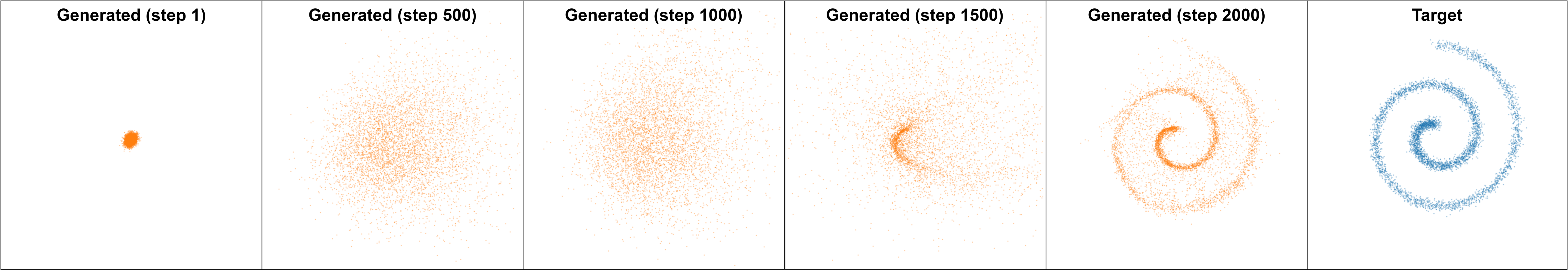}
    \caption{$\mathcal{L}^2$ distance(RBF kernel)}
\end{subfigure}

\vspace{0.4cm}

\begin{subfigure}{0.8\textwidth}
    \centering
    \includegraphics[width=\textwidth]{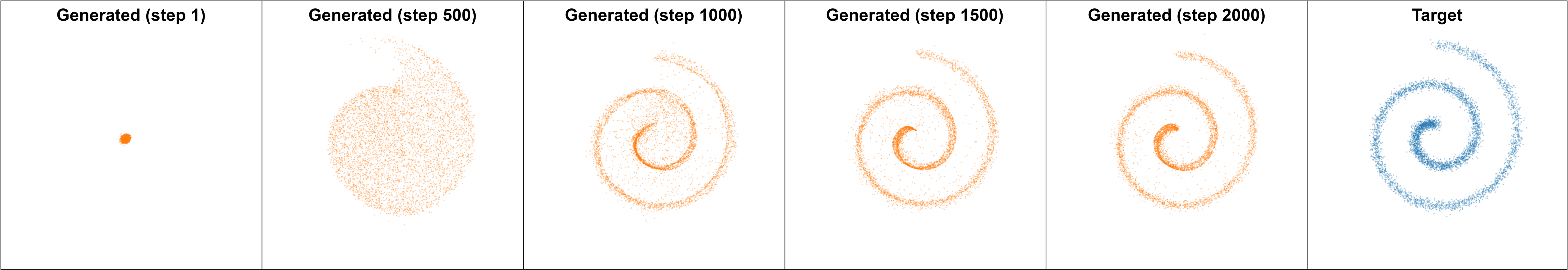}
    \caption{Forward KL divergence (Laplace kernel, original drifting model)}
\end{subfigure}

\vspace{0.4cm}

\begin{subfigure}{0.8\textwidth}
    \centering
    \includegraphics[width=\textwidth]{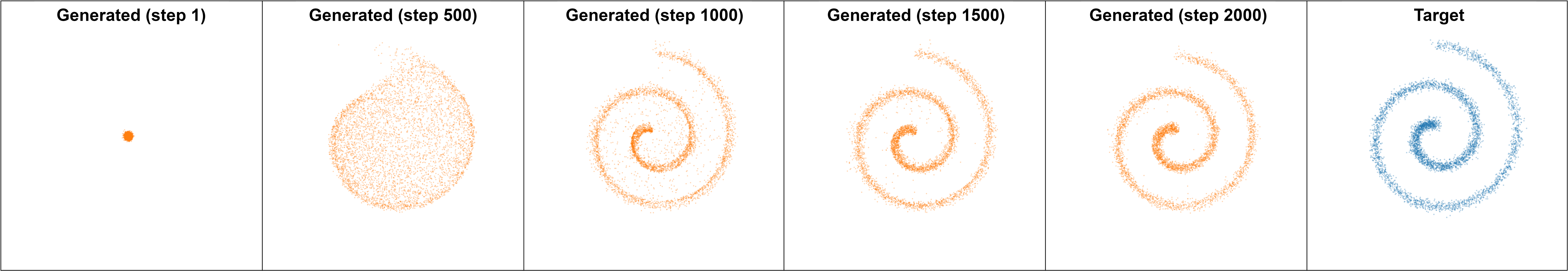}
    \caption{Forward KL divergence (RBF kernel)}
\end{subfigure}

\caption{Training results with the velocity field of gradient flow under different implementations of divergence and kernel function on 2D-toy dataset.}
\label{fig:swiss_images}
\end{figure}

\section{Conclusion and Discussion}
\label{sec:conclusion}

We have found a new family of generative models, and given a mathematical equivalence between the Drifting Model and the Wasserstein gradient flow of the KL divergence under KDE approximation as a special case of our method, \textbf{Gradient Flow Drifting}. We have proved that under the aid of a finely designed kernel function, matching the pushforward distribution of KDE can achieve an approximation of the original distribution, and then extended this method onto Riemannian manifold which loosens the constraints on the kernel function, and make this method more suitable for the semantic space which is used in Drifting Model~\cite{deng2026generative}. Besides, we did some preliminary experiments on synthetic benchmarks to validate the framework.

\paragraph{Limitations.}
Our approach utilizes the convergence of the KDE distribution to induce the convergence of the original distribution. However, in practice, we can only approximate this KDE distribution using minibatches. As the dimension increases, the variance of the minibatch estimation will gradually increase, which will seriously affect the stability of the training and the final convergence effect, as other kernel-based methods suffer.

\paragraph{Future work.}
Future work includes extending Gradient Flow Drifting to large-scale, high-dimensional datasets and diverse generation tasks, such as conditional generation and multi-modal generation. We plan to conduct comprehensive ablation experiments to evaluate the contribution of the combined reverse KL and $\chi^2$ gradient flows, explore the impact of different kernel functions and bandwidth choices and investigate acceleration techniques such as mini-batch particle updates and kernel approximation to improve computational efficiency and practical application capabilities. Meanwhile, we will follow the engineering techniques used in drifting model~\cite{deng2026generative}, like training the model in semantic space and using multiple bandwidths. Furthermore, as theoretical analysis indicates, the Riemannian manifold is highly suitable for our approach. We will employ the hyperspherical semantic space constructed by JEPA, use ViT-based instead of a CNN-based architecture to achieve high computation efficiency, and make this type of model more scalable.


\bibliographystyle{plainnat}
\bibliography{references}

\newpage
\appendix
\section*{Appendix}

\section{Definitions and Standing Assumptions}
\label{app:assumptions}

\paragraph{Probability measures.}
Let $\mathcal{P}(\R^d)$ denote the set of Borel probability measures on
$\R^d$. For $\mu\in\mathcal{P}(\R^d)$ and a measurable function $g$,
we write $\E_\mu[g]=\int_{\R^d}g(\by)\mathrm{d}\mu(\by)$ whenever the
integral is well-defined.

\begin{assumption}[Kernel Regularity \textbf{K1}--\textbf{K4}]
\label{ass:kernel}
Let $k:\R^d\times\R^d\to\R$ be a kernel satisfying:
\begin{enumerate}[nosep, label=\textbf{K\arabic*.}]
  \item\label{K1} \textbf{Characteristic.}\;
    The kernel $k$ is positive definite, and the mean embedding
    $\mu\mapsto\int k(\cdot,\by)\mathrm{d}\mu(\by)$ is injective on
    $\mathcal{P}(\R^d)$.
  \item\label{K2} \textbf{Uniform gradient bound.}\;
    $\displaystyle M_k:=\sup_{\bx,\by\in\R^d}
    \bigl\|\nabla_\bx k(\bx,\by)\bigr\|<\infty$.
  \item\label{K3} \textbf{Strict positivity.}\;
    $k(\bx,\by)>0$ for all $\bx,\by\in\R^d$.
  \item\label{K4} \textbf{Differentiability.}\;
    For every $\by\in\R^d$, the map
    $\bx\mapsto k(\bx,\by)$ is continuously differentiable.
\end{enumerate}
\end{assumption}

\begin{remark}[Normalized KDE]
\label{rem:normalization}
For a translation-invariant kernel $k(\bx,\by)=\varphi(\bx-\by)$
with $\varphi\in L^1(\R^d)$, $\int\mu_\mathrm{kde}(\bx)\mathrm{d}\bx
=\|\varphi\|_{L^1}=:Z_k$ for every $\mu\in\mathcal{P}(\R^d)$.
The normalized density $\bar\mu_h:=\mu_\mathrm{kde}/Z_k\in\mathcal{P}(\R^d)$
satisfies $\nabla\log\bar\mu_h=\nabla\log\mu_\mathrm{kde}$, so all
score-based formulas are unaffected by the normalization constant.
\end{remark}

\section{Injectivity of the KDE Operator}
\label{app:injectivity}

We show that a characteristic kernel allows the recovery of the
original measure from its KDE.

\begin{proposition}[KDE Injectivity]
\label{prop:injectivity}
Let $k$ satisfy~\ref{K1}. If\/ $\mu_\mathrm{kde}(\bx)=\nu_\mathrm{kde}(\bx)$
for all $\bx\in\R^d$, then $\mu=\nu$.
\end{proposition}

\begin{proof}
Let $\mathcal{H}_k$ be the RKHS of $k$ and denote the mean embeddings
$m_\mu:=\int k(\cdot,\by)\mathrm{d}\mu(\by)$,
$m_\nu:=\int k(\cdot,\by)\mathrm{d}\nu(\by)\in\mathcal{H}_k$.
By the reproducing property,
$m_\mu(\bx)=\langle m_\mu,k(\cdot,\bx)\rangle_{\mathcal{H}_k}
=\mu_\mathrm{kde}(\bx)$
and similarly for $\nu$. Hence $\mu_\mathrm{kde}=\nu_\mathrm{kde}$ pointwise
implies $\langle m_\mu-m_\nu,k(\cdot,\bx)\rangle_{\mathcal{H}_k}=0$ for all
$\bx$. Since $\{k(\cdot,\bx):\bx\in\R^d\}$ spans a dense subset
of $\mathcal{H}_k$, we obtain $m_\mu=m_\nu$ in $\mathcal{H}_k$.
By~\ref{K1} (injectivity of the mean embedding), $\mu=\nu$.
\end{proof}

\section{Regularity of KDE-Smoothed Densities}
\label{app:regularity}

We establish that the smoothness of $\mu_\mathrm{kde}$ is inherited
entirely from the kernel, regardless of the regularity of $\mu$.

\begin{theorem}[KDE Regularity]
\label{thm:regularity}
Let $k$ satisfy~\ref{K2}--\ref{K4} and let $\mu\in\mathcal{P}(\R^d)$
be arbitrary. Then:
\begin{enumerate}[nosep, label=(\roman*)]
  \item\label{item:C1}
    $\mu_\mathrm{kde}\in C^1(\R^d)$ and differentiation commutes with
    integration:
    \begin{equation}\label{eq:leibniz}
      \nabla_\bx\,\mu_\mathrm{kde}(\bx)
      =\int_{\R^d}\nabla_\bx k(\bx,\by)\mathrm{d}\mu(\by).
    \end{equation}
  \item\label{item:pos}
    $\mu_\mathrm{kde}(\bx)>0$ for all $\bx\in\R^d$.
  \item\label{item:bdd-grad}
    $\sup_\bx\|\nabla\mu_\mathrm{kde}(\bx)\|\le M_k<\infty$.
\end{enumerate}
\end{theorem}

\begin{proof}
\ref{item:C1}:\;
By~\ref{K4}, $\bx\mapsto k(\bx,\by)$ is $C^1$ for each $\by$.
By~\ref{K2}, $\|\nabla_\bx k(\bx,\by)\|\le M_k$ for all
$\bx,\by$. Since the constant $M_k$ is trivially
$\mu$-integrable ($\int M_k\mathrm{d}\mu=M_k<\infty$ for any
probability measure $\mu$), the Leibniz integral rule yields
\eqref{eq:leibniz} and continuity of the derivative.

\ref{item:pos}:\;
By~\ref{K3}, $k(\bx,\by)>0$ for all $\bx,\by$.
Since $\mu$ is a nonzero positive measure,
$\mu_\mathrm{kde}(\bx)=\int k(\bx,\by)\mathrm{d}\mu(\by)>0$.

\ref{item:bdd-grad}:\;
$\|\nabla\mu_\mathrm{kde}(\bx)\|
 =\bigl\|\int\nabla_\bx k(\bx,\by)\mathrm{d}\mu(\by)\bigr\|
 \le\int\|\nabla_\bx k(\bx,\by)\|\mathrm{d}\mu(\by)
 \le M_k$.
\end{proof}

\begin{corollary}
\label{cor:wgf-applies}
For any $p,q\in\mathcal{P}(\R^d)$, the KDE densities $p_\mathrm{kde},q_\mathrm{kde}$ are
strictly positive and $C^1$. In particular, the log-ratio
$\log(p_\mathrm{kde}/q_\mathrm{kde})$ is well-defined and $C^1$, and the
$f$-divergence machinery of the next section applies to
$p_\mathrm{kde},q_\mathrm{kde}$ with no additional assumptions on $p,q$.
\end{corollary}

\section{Wasserstein Gradient Flows of $f$-Divergences}
\label{app:f-div}

We recall the Wasserstein gradient flow (WGF)
framework~\cite{ambrosio2005gradient}
for $f$-divergences between smooth positive densities.

\begin{definition}[$f$-divergence]
\label{def:f-div}
Let $f:(0,\infty)\to\R$ be convex with $f(1)=0$.
For positive densities $\rho,\pi$ on $\R^d$,
\begin{equation}\label{eq:f-div}
  D_f(\rho\|\pi)
  :=\int_{\R^d}\pi(\bx)\,
  f\!\left(\frac{\rho(\bx)}{\pi(\bx)}\right)\mathrm{d}\bx.
\end{equation}
\end{definition}

\begin{proposition}[First Variation, Velocity, and Energy Dissipation]
\label{prop:wgf}
Let $\rho,\pi\in C^1(\R^d)$ with $\rho,\pi>0$ everywhere and
$D_f(\rho\|\pi)<\infty$. Then:
\begin{enumerate}[nosep, label=(\roman*)]
  \item The first variation of $D_f(\rho\|\pi)$ with respect to
    $\rho$ is
    \begin{equation}\label{eq:first-var}
      \frac{\delta D_f(\rho\|\pi)}{\delta\rho}(\bx)
      =f'\!\left(\frac{\rho(\bx)}{\pi(\bx)}\right).
    \end{equation}
  \item The WGF particle velocity is
    \begin{equation}\label{eq:wgf-v}
      \mathbf{v}_f(\bx)
      =-\nabla f'\!\left(\frac{\rho(\bx)}{\pi(\bx)}\right).
    \end{equation}
  \item Suppose in addition that the boundary terms arising
    from integration by parts vanish:
    \begin{equation}\label{eq:boundary-vanish}
      \lim_{R\to\infty}\int_{\|\bx\|=R}
      \rho_t\,f'\!\bigl(\tfrac{\rho_t}{\pi}\bigr)\,
      \nabla f'\!\bigl(\tfrac{\rho_t}{\pi}\bigr)
      \cdot\hat{\mathbf{n}}\,dS=0.
    \end{equation}
    Then along any smooth WGF solution $(\rho_t)_{t\ge0}$:
    \begin{equation}\label{eq:dissipation}
      \frac{d}{dt}D_f(\rho_t\|\pi)
      =-\int_{\R^d}\rho_t(\bx)
       \left\|\nabla f'\!\left(
       \frac{\rho_t(\bx)}{\pi(\bx)}\right)\right\|^2\mathrm{d}\bx
      \le 0.
    \end{equation}
\end{enumerate}
\end{proposition}

\begin{proof}
\emph{(i)}\;
Let $\eta$ be a smooth, compactly-supported perturbation with
$\int\eta\mathrm{d}\bx=0$. Set $u:=\rho/\pi$. Then
\[
  \left.\frac{d}{d\epsilon}\right|_{\epsilon=0}
  D_f(\rho+\epsilon\eta\|\pi)
  =\int\pi\,f'(u)\,\frac{\eta}{\pi}\mathrm{d}\bx
  =\int f'(u)\,\eta\mathrm{d}\bx,
\]
giving~\eqref{eq:first-var}.

\emph{(ii)}\;
Immediate from Definition~\ref{def:wgf}:
$\mathbf{v}_f=-\nabla\frac{\delta\mathcal{F}}{\delta\rho}=-\nabla f'(u)$.

\emph{(iii)}\;
Write $\Phi:=f'(\rho_t/\pi)=\frac{\delta\mathcal{F}}{\delta\rho}$.
Using the continuity equation~\eqref{eq:continuity}:
\[
  \frac{d}{dt}D_f(\rho_t\|\pi)
  =\int_{\R^d}\Phi\,\partial_t\rho_t\mathrm{d}\bx
  =\int_{\R^d}\Phi\,
   \nabla\cdot\!\bigl(\rho_t\nabla\Phi\bigr)\mathrm{d}\bx.
\]
The product rule gives
$\Phi\,\nabla\cdot(\rho_t\nabla\Phi)
=\nabla\cdot(\Phi\,\rho_t\nabla\Phi)
-\rho_t\|\nabla\Phi\|^2$.
Integrating the divergence term over $B_R$ and applying the
divergence theorem yields the boundary integral
\[
  \int_{\|\bx\|=R}\rho_t\,\Phi\,
  (\nabla\Phi\cdot\hat{\mathbf{n}})\,dS,
\]
which vanishes as $R\to\infty$
by~\eqref{eq:boundary-vanish}. Hence
\[
  \frac{d}{dt}D_f(\rho_t\|\pi)
  =-\int_{\R^d}\rho_t\|\nabla\Phi\|^2\mathrm{d}\bx\le 0.
  \qedhere
\]
\end{proof}

\begin{remark}[Boundary condition~\eqref{eq:boundary-vanish}]
\label{rem:boundary}
We distinguish two settings and don't assume the vanishing-boundary condition~\eqref{eq:boundary-vanish}. Within $\R^d$, in our context we assume the original $p,q\in\mathcal{P}_2(\R^d)$ and the induced KDE distributions can be easily satisfy the vanishing-boundary condition.

\begin{enumerate}[leftmargin=2em]
  \item[\textbf{$\R^d$.}]
    Condition~\eqref{eq:boundary-vanish} is an assumption on the
    joint tail behavior of $\rho_t$ and $\pi$. It holds whenever
    $\rho_t$, $\pi$, their ratio, and its gradient have sufficient
    decay at infinity. In the context of this paper, $\rho_t$ and
    $\pi$ are KDE-smoothed densities, whose tails are governed by
    the kernel. For any kernel with exponential or faster decay
    (e.g., Gaussian, Mat\'ern with $\nu>1$, Pseudo-Huber), the
    KDE-smoothed densities inherit exponential decay and the
    condition is satisfied for all $p,q\in\mathcal{P}_2(\R^d)$ with finite
    second moments. For kernels with only polynomial decay
    (e.g., IMQ with exponent $\beta$), the condition requires
    $\beta>(d-2)/2$.

  \item[\textbf{$\Sph^{d-1}$.}]
    If the ambient space is a compact Riemannian manifold $M$
    without boundary,
    then condition~\eqref{eq:boundary-vanish} is
    \emph{vacuous}. Indeed, the divergence theorem on $M$ reads
    \[
      \int_M\nabla_g\cdot X\,d\mathrm{vol}_g
      =\int_{\partial M}g(X,\hat{\mathbf{n}})\,dS=0
    \]
    since $\partial M=\varnothing$. Consequently, the energy
    dissipation inequality~\eqref{eq:dissipation} holds
    unconditionally for any $f$-divergence, any
    kernel satisfying the manifold analogues
    of~\ref{K1}--\ref{K4}, and any $p,q\in\mathcal{P}(M)$.
\end{enumerate}
\end{remark}

\begin{proposition}[Identifiability for $f$-divergence gradient flows]
\label{prop:ident-general}
Let $M$ be a connected Riemannian manifold (e.g., $\R^d$ or
$\Sph^{d-1}$). Let $\rho,\pi\in C^1(M)$ with $\rho,\pi>0$
and $\int_M\rho\,d\mathrm{vol}=\int_M\pi\,d\mathrm{vol}$.
If $f$ is strictly convex and the WGF velocity vanishes
identically, $\mathbf{v}_f\equiv\mathbf{0}$, then $\rho=\pi$.
\end{proposition}

\begin{proof}
By~\eqref{eq:wgf-v}, $\mathbf{v}_f\equiv\mathbf{0}$ implies
$\nabla f'(\rho/\pi)\equiv\mathbf{0}$ on $M$.
Since $M$ is connected and $f'(\rho/\pi)\in C^0(M)$,
we have $f'(\rho/\pi)\equiv c$ for some constant $c$.
Strict convexity of $f$ implies $f'$ is strictly monotone,
so $\rho/\pi\equiv(f')^{-1}(c)=:\lambda>0$.
Integrating: $\int\rho=\lambda\int\pi$, hence $\lambda=1$
and $\rho=\pi$.
\end{proof}

\begin{remark}[Specific $f$-divergences]
\label{rem:table-f}
Writing $u:=\rho(\bx)/\pi(\bx)$, we record the velocities for
the divergences of primary interest:
\begin{center}
\renewcommand{\arraystretch}{1.15}
\begin{tabular}{@{} llll @{}}
  \toprule
  $f(u)$ & Divergence & $f'(u)$ & Velocity $\mathbf{v}_f(\bx)$\\
  \midrule
  $u\log u$ & $\KL(\rho\|\pi)$ & $1+\log u$
    & $\nabla\log\pi-\nabla\log\rho$\\
  $-\log u$ & $\KL(\pi\|\rho)$ & $-1/u$
    & $\frac{\pi}{\rho}(\nabla\log\pi-\nabla\log\rho)$\\
  $\frac{1}{2}(u-1)^2$ & $\chi^2(\rho\|\pi)$ & $(u-1)$
    & $\frac{\rho}{\pi}(\nabla\log\pi-\nabla\log\rho)$\\
  \bottomrule
\end{tabular}
\end{center}
\end{remark}

\section{KDE Score Formula}
\label{app:score}

We express the score of $\mu_\mathrm{kde}$ as an expectation under $\mu$,
establishing sample computability.

\begin{proposition}[KDE Score]
\label{prop:score}
Let $k$ satisfy~\ref{K2}--\ref{K4} and $\mu\in\mathcal{P}(\R^d)$.
Then for all $\bx\in\R^d$:
\begin{equation}\label{eq:score}
  \nabla\log\mu_\mathrm{kde}(\bx)
  =\frac{\int\nabla_\bx k(\bx,\by)\mathrm{d}\mu(\by)}
        {\int k(\bx,\by)\mathrm{d}\mu(\by)}.
\end{equation}
\end{proposition}

\begin{proof}
By Theorem~\ref{thm:regularity}\ref{item:C1}
and~\ref{item:pos},
$\mu_\mathrm{kde}\in C^1$ and $\mu_\mathrm{kde}>0$, so
$\nabla\log\mu_\mathrm{kde}=\nabla\mu_\mathrm{kde}/\mu_\mathrm{kde}$.
Substituting~\eqref{eq:leibniz} gives~\eqref{eq:score}.
\end{proof}

\begin{corollary}[Gaussian Kernel Score]
\label{cor:gauss-score}
For $k_h(\bx,\by)=\exp\!\bigl(-\|\bx-\by\|^2/(2h^2)\bigr)$:
\begin{equation}\label{eq:gauss-score}
  \nabla\log\mu_\mathrm{kde}(\bx)
  =\frac{1}{h^2}\bigl(\E_{\mu,k_h}[\by\mid\bx]-\bx\bigr),
\end{equation}
where $\E_{\mu,k_h}[\by\mid\bx]
:=\int\by\,k_h(\bx,\by)\mathrm{d}\mu(\by)\big/
\int k_h(\bx,\by)\mathrm{d}\mu(\by)$.
\end{corollary}

\begin{proof}
For the Gaussian kernel,
$\nabla_\bx k_h(\bx,\by)
=\frac{\by-\bx}{h^2}k_h(\bx,\by)$.
Substituting into~\eqref{eq:score}:
\[
  \nabla\log\mu_\mathrm{kde}(\bx)
  =\frac{\int\frac{\by-\bx}{h^2}k_h(\bx,\by)\mathrm{d}\mu(\by)}
        {\int k_h(\bx,\by)\mathrm{d}\mu(\by)}
  =\frac{1}{h^2}
   \!\left(\!\frac{\int\by\,k_h(\bx,\by)\mathrm{d}\mu(\by)}
                  {\int k_h(\bx,\by)\mathrm{d}\mu(\by)}
   -\bx\!\right).
\]
\end{proof}

\section{Identifiability}
\label{app:ident-conv}

Throughout this section, $\pi := p_\mathrm{kde}$ and $\rho_0 := q_{0,\mathrm{kde}}$
denote the KDE-smoothed densities of the target distribution $p$
and the initial generated distribution $q_0$, respectively.
By Theorem~\ref{thm:regularity}, both are strictly positive and
$C^1$ under Assumptions~\ref{K1}--\ref{K4}, with no regularity
conditions on $p$ or $q_0$. We consider the Wasserstein-2 gradient
flow of the $f$-divergence $\mathcal{F}[\rho] = D_f(\rho\|\pi)$,
i.e., the continuity equation
\begin{equation}\label{eq:pde-flow}
    \partial_t \rho_t = \nabla \cdot \bigl(\rho_t \nabla \Phi_t\bigr),
    \qquad
    \Phi_t := f'\!\left(\frac{\rho_t}{\pi}\right),
\end{equation}
with initial condition $\rho_0 > 0$.

\begin{remark}[Density-level vs.\ particle-level flow]
\label{rem:density-vs-particle}
Equation~\eqref{eq:pde-flow} describes the evolution of a smooth
density $\rho_t$ in the Wasserstein-2 metric. In the practical
training algorithm \ref{sec:pipeline}, one evolves a finite
collection of particles whose empirical distribution is $q_t$,
and estimates the velocity using the KDE density
$q_{t,\mathrm{kde}}$ from a mini-batch. The convergence theorems
below apply to the idealized density-level flow;
the particle system is viewed as a consistent approximation
that converges to this flow in the large-sample limit.
\end{remark}

\subsection{Proof of Identifiability (Theorem~\ref{thm:identifiability})}
\label{app:proof-ident}

\begin{proof}
The proof proceeds in two steps.

\medskip
\noindent\textbf{Step 1: Constant density ratio.}\;
By definition, $\mathbf{v}_f^{\mathrm{kde}}(\bx)
= -\nabla f'\!\bigl(q_\mathrm{kde}(\bx)/p_\mathrm{kde}(\bx)\bigr)$.
The hypothesis $\mathbf{v}_f^{\mathrm{kde}} \equiv \mathbf{0}$
thus implies
\begin{equation}\label{eq:ident-grad-zero}
    \nabla f'\!\left(\frac{q_\mathrm{kde}(\bx)}{p_\mathrm{kde}(\bx)}\right) = \mathbf{0}
    \quad \text{for all } \bx.
\end{equation}
Let $M$ be the ambient space ($\R^d$ or a connected Riemannian
manifold). By Theorem~\ref{thm:regularity}, $p_\mathrm{kde}$ and
$q_\mathrm{kde}$ are $C^1$ with $p_\mathrm{kde}, q_\mathrm{kde} > 0$,
so the ratio $u(\bx) := q_\mathrm{kde}(\bx)/p_\mathrm{kde}(\bx)$ is
$C^1$ and strictly positive. By the composition rule,
$\Phi := f' \circ u \in C^1(M)$.
Since $M$ is connected and $\nabla \Phi \equiv \mathbf{0}$,
$\Phi$ is a constant: $f'(u(\bx)) = c$ for some $c \in \R$.

Strict convexity of $f$ implies that $f'$ is strictly monotone, hence
injective. Therefore $u(\bx) = (f')^{-1}(c) =: \lambda > 0$
for all $\bx$, i.e., $q_\mathrm{kde} = \lambda\, p_\mathrm{kde}$
everywhere.

\medskip
\noindent\textbf{Step 2: KDE matching implies distribution matching.}\;
Integrating both sides over $M$:
\[
    \int_M q_\mathrm{kde}\,\mathrm{d}\bx
    = \lambda \int_M p_\mathrm{kde}\,\mathrm{d}\bx.
\]
For a translation-invariant kernel,
$\int q_\mathrm{kde} = \int p_\mathrm{kde} = Z_k$
(Remark~\ref{rem:normalization}); on a compact manifold the same
equality holds by symmetry. Hence $\lambda = 1$ and
$q_\mathrm{kde} = p_\mathrm{kde}$. By the injectivity of the KDE
operator under characteristic kernels
(Proposition~\ref{prop:injectivity}), $q = p$.
\end{proof}

\begin{remark}[MMD identifiability]
\label{rem:mmd-ident}
For the MMD functional (Proposition~\ref{prop:mmd}),
$\mathbf{v}_\MMD \equiv \mathbf{0}$ implies
$\nabla(p_\mathrm{kde} - q_\mathrm{kde}) \equiv \mathbf{0}$.
By connectedness, $p_\mathrm{kde} - q_\mathrm{kde} = c$. Integrating
gives $c = 0$, so $p_\mathrm{kde} = q_\mathrm{kde}$, and
Proposition~\ref{prop:injectivity} yields $p = q$.
\end{remark}

\section{Core Equivalence: Drifting as KL Gradient Flow}
\label{app:equivalence}

\begin{theorem}[Core Equivalence]
\label{thm:core}
Let $k_h$ be the Gaussian kernel and
$p,q\in\mathcal{P}(\R^d)$. Then the drifting
field~\eqref{eq:drifting-field-original} satisfies
\begin{equation}\label{eq:core}
  \boxed{
  \mathbf{V}_{p,q}(\bx)
  =h^2\bigl(\nabla\log p_\mathrm{kde}(\bx)
            -\nabla\log q_\mathrm{kde}(\bx)\bigr)
  =h^2\,\mathbf{v}_\KL(\bx)\big|_{\rho=q_\mathrm{kde},\;\pi=p_\mathrm{kde}},
  }
\end{equation}
where $\mathbf{v}_\KL=\nabla\log\pi-\nabla\log\rho$ is the WGF velocity
of the KL divergence $D_\KL(\rho\|\pi)$
(Remark~\ref{rem:table-f}).
\end{theorem}

\begin{proof}
Apply Corollary~\ref{cor:gauss-score} to $p$ and $q$ respectively:
\begin{align}
  h^2\nabla\log p_\mathrm{kde}(\bx)
  &=\E_{p,k_h}[\by\mid\bx]-\bx,
  \label{eq:p-score}\\
  h^2\nabla\log q_\mathrm{kde}(\bx)
  &=\E_{q,k_h}[\by\mid\bx]-\bx.
  \label{eq:q-score}
\end{align}
Subtracting, the $\bx$ terms cancel:
\begin{equation}
  h^2\bigl(\nabla\log p_\mathrm{kde}(\bx)
           -\nabla\log q_\mathrm{kde}(\bx)\bigr)
  =\E_{p,k_h}[\by\mid\bx]
  -\E_{q,k_h}[\by\mid\bx]
  =\mathbf{V}_{p,q}(\bx),
\end{equation}
The identification with $\mathbf{v}_\KL$ follows from
Remark~\ref{rem:table-f} with $\rho=q_\mathrm{kde}$, $\pi=p_\mathrm{kde}$
(both smooth and positive by Corollary~\ref{cor:wgf-applies}).
\end{proof}

\section{MMD Gradient Flow}
\label{app:mmd}

\begin{definition}[Squared MMD]
\label{def:mmd}
For a positive-definite kernel $k$ with RKHS $\mathcal{H}_k$,
the squared MMD between $p,q\in\mathcal{P}(\R^d)$ is
\begin{equation}
  \MMD_k^2(q,p):=\frac{1}{2}\|m_q-m_p\|_{\mathcal{H}_k}^2,
\end{equation}
where $m_\mu:=\int k(\cdot,\by)\mathrm{d}\mu(\by)$ is the mean embedding.
\end{definition}

\begin{proposition}[MMD Flow Velocity]
\label{prop:mmd}
Let $k$ satisfy~\ref{K1}--\ref{K4}. The WGF of
$\mathcal{F}[q]:=\frac{1}{2}\MMD_k^2(q,p)$ has particle velocity
\begin{equation}\label{eq:mmd-v}
  \mathbf{v}_\MMD(\bx)
  =\int\nabla_\bx k(\bx,\by)\mathrm{d}(p-q)(\by)
  =\nabla\bigl(p_\mathrm{kde}(\bx)-q_\mathrm{kde}(\bx)\bigr).
\end{equation}
\end{proposition}

\begin{proof}
Perturb $q\to(1-\epsilon)q+\epsilon\delta_\bx$, so
$m_{q_\epsilon}=m_q+\epsilon(k(\cdot,\bx)-m_q)$.
Expanding and differentiating at $\epsilon=0$:
$$
  \frac{\delta\mathcal{F}}{\delta q}(\bx)
  =\langle m_q-m_p,\,k(\cdot,\bx)\rangle_{\mathcal{H}_k}
  =m_q(\bx)-m_p(\bx)
  =q_\mathrm{kde}(\bx)-p_\mathrm{kde}(\bx),
$$
using the reproducing property. The velocity is
$\mathbf{v}_\MMD=-\nabla\frac{\delta\mathcal{F}}{\delta q}
=\nabla(p_\mathrm{kde}-q_\mathrm{kde})$.
By Theorem~\ref{thm:regularity}\ref{item:C1},
differentiation under the integral gives the first equality
in~\eqref{eq:mmd-v}.
\end{proof}

\section{Superposition of Gradient Flows}
\label{app:superposition}

\begin{proposition}[Superposition]
\label{prop:super}
Let $\mathcal{F}_1,\mathcal{F}_2:\mathcal{P}(\R^d)\to\R$ be functionals with
well-defined first variations.
For any $\alpha,\beta\ge0$, the mixed functional
$\mathcal{F}_\mathrm{mix}:=\alpha\mathcal{F}_1+\beta\mathcal{F}_2$ has WGF velocity
\begin{equation}\label{eq:super}
  \mathbf{v}_\mathrm{mix}(\bx)=\alpha\,\mathbf{v}_1(\bx)+\beta\,\mathbf{v}_2(\bx),
\end{equation}
where $\mathbf{v}_i=-\nabla\frac{\delta\mathcal{F}_i}{\delta\rho}$.
\end{proposition}

\begin{proof}
By linearity of the first variation and the gradient:
$$
  \mathbf{v}_\mathrm{mix}
  =-\nabla\frac{\delta\mathcal{F}_\mathrm{mix}}{\delta\rho}
  =-\nabla\!\left(\alpha\frac{\delta\mathcal{F}_1}{\delta\rho}
   +\beta\frac{\delta\mathcal{F}_2}{\delta\rho}\right)
  =\alpha\mathbf{v}_1+\beta\mathbf{v}_2.
$$
\end{proof}

\paragraph{Unified framework.}
Table~\ref{tab:unified} summarizes the KDE-based gradient flow
velocities, all sample-computable via the score
formula~\eqref{eq:score}.

\begin{table}[h]
\caption{Unified KDE-based gradient flow framework.
All velocities are expressed via $p_\mathrm{kde},q_\mathrm{kde}$ and are
sample-computable.}
\label{tab:unified}
\centering
\renewcommand{\arraystretch}{1.15}
\begin{tabular}{@{} lll @{}}
  \toprule
  Functional & WGF velocity $\mathbf{v}(\bx)$ & Model\\
  \midrule
  $D_\KL(\bar q_h\|\bar p_h)$
    & $\nabla\log p_\mathrm{kde}-\nabla\log q_\mathrm{kde}$
    & \textbf{Drifting} ($/h^2$)\\
  $D_\KL(\bar p_h\|\bar q_h)$
    & $\nabla(p_\mathrm{kde}/q_\mathrm{kde})$
    & - \\
  $\chi^2(\bar q_h\|\bar p_h)$
    & $-\nabla(q_\mathrm{kde}/p_\mathrm{kde})$
    & - \\
  $\MMD_{k_h}^2(q,p)$
    & $\nabla(p_\mathrm{kde}-q_\mathrm{kde})$
    & MMD generators\\
  \bottomrule
\end{tabular}
\end{table}

\section{Analysis of Specific Kernel Families}
\label{app:kernels}

We verify Assumptions~\ref{K1}--\ref{K4} for several kernel families.
For each kernel, we also compute the \emph{score weight}
$w(r)$ defined by
$\nabla_\bx\log k(\bx,\by)=-w(r)(\bx-\by)$ where
$r:=\|\bx-\by\|$.

\subsection{Euclidean Kernels}
\label{app:kernels-euclidean}

\subsubsection{Gaussian (RBF) Kernel}

\begin{proposition}
\label{prop:gauss}
The Gaussian kernel
$k_h(\bx,\by):=\exp\!\bigl(-\|\bx-\by\|^2/(2h^2)\bigr)$
satisfies~\ref{K1}--\ref{K4}. Its score weight is
$w(r)=1/h^2$ (constant).

If $k_h$ is the Gaussian kernel, one additionally obtains
$\mu_\mathrm{kde}\in C^\infty(\R^d)$ with every derivative uniformly
bounded for any $\mu\in\mathcal{P}(\R^d)$.
\end{proposition}

\begin{proof}
\ref{K1}:\;
The Fourier transform
$\hat\varphi(\bm\omega)=(2\pi h^2)^{d/2}
\exp(-h^2\|\bm\omega\|^2/2)>0$ for all $\bm\omega$.
By~\citet[Theorem~9]{sriperumbudur2010hilbert}, a
translation-invariant kernel with strictly positive Fourier
transform is characteristic.

\ref{K2}:\;
$\|\nabla_\bx k_h\|
=\frac{r}{h^2}e^{-r^2/(2h^2)}$
where $r=\|\bx-\by\|$. Maximizing over $r\ge0$:
the maximum occurs at $r=h$, giving
$M_k=\frac{1}{h}e^{-1/2}=\frac{1}{h\sqrt{e}}<\infty$.

\ref{K3}:\; $\exp(\cdot)>0$.
\quad
\ref{K4}:\; $k_h\in C^\infty(\R^d\times\R^d)$.
\end{proof}

\subsubsection{Mat\'ern-$\nu$ Kernel}

\begin{proposition}
\label{prop:matern}
The Mat\'ern kernel with smoothness $\nu>0$ and length scale
$\ell>0$,
\[
  k_{\nu,\ell}(\bx,\by)
  =\frac{2^{1-\nu}}{\Gamma(\nu)}
   \left(\frac{\sqrt{2\nu}\,r}{\ell}\right)^{\!\nu}
   K_\nu\!\left(\frac{\sqrt{2\nu}\,r}{\ell}\right),
   \qquad r:=\|\bx-\by\|,
\]
where $K_\nu$ is the modified Bessel function of the second kind,
satisfies all four assumptions if and only if $\nu>1$.
\end{proposition}

\begin{proof}
\ref{K1}:\;
The Fourier transform
$\hat\varphi(\bm\omega)
\propto(2\nu/\ell^2+\|\bm\omega\|^2)^{-(\nu+d/2)}>0$ for all
$\bm\omega$ and $\nu>0$, so $k_{\nu,\ell}$ is characteristic.

\ref{K3}:\; $k_{\nu,\ell}>0$ since $K_\nu(z)>0$ for $z>0$ and
$k_{\nu,\ell}(\bx,\bx)=1$.

\ref{K4}:\;
The sample path regularity theory of Mat\'ern processes shows that
$k_{\nu,\ell}$ is $C^{\lceil\nu\rceil-1}$ as a function of $r$.
When $\nu\le1$, a cusp at $r=0$ (Laplace-like) violates~\ref{K4}.
When $\nu>1$, at least $C^1$ regularity is guaranteed.

\ref{K2}:\;
When $\nu>1$, $\|\nabla_\bx k_{\nu,\ell}\|$ is continuous
(by~\ref{K4}) and decays exponentially as $r\to\infty$, hence
bounded. When $\nu\le1$, the gradient diverges at $r=0$.
\end{proof}

\subsubsection{Summary of Euclidean Kernels}

\begin{table}[h]
\caption{Verification of~\ref{K1}--\ref{K4} for Euclidean kernels.}
\label{tab:euclidean}
\centering
\renewcommand{\arraystretch}{1.15}
\begin{tabular}{@{} lccccc @{}}
  \toprule
  Kernel & \ref{K1} & \ref{K2} & \ref{K3} & \ref{K4} & Status\\
  \midrule
  Gaussian
    & \checkmark & \checkmark & \checkmark & \checkmark
    & \checkmark\\
  IMQ
    & \checkmark & \checkmark & \checkmark & \checkmark
    & \checkmark\\
  Pseudo-Huber
    & \checkmark & \checkmark & \checkmark & \checkmark
    & \checkmark\\
  Mat\'ern ($\nu>1$)
    & \checkmark & \checkmark & \checkmark & \checkmark
    & \checkmark\\
  Laplace
    & \checkmark & \checkmark & \checkmark & \ding{55}
    & \ding{55}\\
  \bottomrule
  \multicolumn{6}{@{}l}{\footnotesize
    \checkmark: verified;\;
    \ding{55}: fails.}
\end{tabular}
\end{table}

\subsection{Spherical Kernels}
\label{app:kernels-sphere}

On the unit sphere
$\Sph^{d-1}:=\{\bx\in\R^d:\|\bx\|=1\}$ with the round metric,
we adapt the kernel assumptions.

\paragraph{Adapted assumptions.}
With $\nabla_\Sph$ denoting the Riemannian gradient on $\Sph^{d-1}$:
\begin{itemize}[nosep, leftmargin=2em]
  \item[\textbf{K1}$_\Sph$.] $k$ is characteristic on
    $\mathcal{P}(\Sph^{d-1})$.
  \item[\textbf{K2}$_\Sph$.] $\sup_{\bx,\by\in\Sph^{d-1}}
    \|\nabla_{\Sph,\bx}k(\bx,\by)\|<\infty$.
  \item[\textbf{K3}$_\Sph$.] $k(\bx,\by)>0$ for all
    $\bx,\by\in\Sph^{d-1}$.
  \item[\textbf{K4}$_\Sph$.] $\bx\mapsto k(\bx,\by)$ is $C^1$
    on $\Sph^{d-1}$.
\end{itemize}

\subsubsection{von Mises--Fisher (vMF) Kernel}

\begin{proposition}
\label{prop:vmf}
The vMF kernel $k_\kappa(\bx,\by):=\exp(\kappa\,\bx^\top\by)$
with $\kappa>0$ satisfies \textbf{K1}$_\Sph$--\textbf{K4}$_\Sph$.
\end{proposition}

\begin{proof}
\textbf{K1}$_\Sph$:\;
The Mercer expansion
$k_\kappa(\bx,\by)=\sum_{\ell=0}^\infty a_\ell(\kappa)
\sum_m Y_\ell^m(\bx)\overline{Y_\ell^m(\by)}$
has coefficients $a_\ell(\kappa)>0$ for all $\ell\ge0$. Since all eigenvalues are positive,
$k_\kappa$ is universal and hence characteristic.

\textbf{K3}$_\Sph$:\; $\exp(\kappa\,\bx^\top\by)>0$.
\quad
\textbf{K4}$_\Sph$:\; $k_\kappa\in C^\infty$.

\textbf{K2}$_\Sph$:\;
The Riemannian gradient is
$\nabla_{\Sph,\bx}k_\kappa
=\kappa\,k_\kappa\,\Proj_{T_\bx\Sph}(\by)$
where $\Proj_{T_\bx\Sph}(\by)=\by-(\bx^\top\by)\bx$.
Since $k_\kappa\le e^\kappa$ and
$\|\Proj_{T_\bx\Sph}(\by)\|\le1$:
$\|\nabla_{\Sph,\bx}k_\kappa\|\le\kappa e^\kappa<\infty$.
\end{proof}

\begin{proposition}[Spherical Core Equivalence]
\label{prop:sphere-equiv}
Let $p,q\in\mathcal{P}(\Sph^{d-1})$ and $k_\kappa$ be the vMF kernel.
Define $\mu_\mathrm{kde}(\bx):=\int_{\Sph^{d-1}}k_\kappa(\bx,\by)
\mathrm{d}\mu(\by)$.
\begin{enumerate}[nosep, label=(\roman*)]
\item The spherical KDE score is
  \begin{equation}\label{eq:vmf-score}
    \nabla_\Sph\log\mu_\mathrm{kde}(\bx)
    =\kappa\,\Proj_{T_\bx\Sph}
     \!\bigl(\E_{\mu,k_\kappa}[\by\mid\bx]\bigr).
  \end{equation}
\item The spherical drifting field
  $\mathbf{V}_{p,q}^\Sph(\bx)
  :=\Proj_{T_\bx\Sph}\!\bigl(
  \E_{p,k_\kappa}[\by\mid\bx]
  -\E_{q,k_\kappa}[\by\mid\bx]\bigr)$ satisfies
  \begin{equation}\label{eq:sphere-core}
    \mathbf{V}_{p,q}^\Sph(\bx)
    =\frac{1}{\kappa}\bigl(
    \nabla_\Sph\log p_\mathrm{kde}(\bx)
    -\nabla_\Sph\log q_\mathrm{kde}(\bx)\bigr).
  \end{equation}
\end{enumerate}
\end{proposition}

\begin{proof}
\emph{(i)}\;
The ambient gradient of $k_\kappa(\bx,\by)$ with respect to $\bx$
is $\kappa\by\,k_\kappa(\bx,\by)$. Hence the ambient gradient of
$\log\mu_\mathrm{kde}$ is
$\kappa\int\by\,k_\kappa\mathrm{d}\mu/\int k_\kappa\mathrm{d}\mu
=\kappa\,\E_{\mu,k_\kappa}[\by\mid\bx]$.
The Riemannian gradient is its tangential projection,
giving~\eqref{eq:vmf-score}.

\emph{(ii)}\;
By linearity of $\Proj_{T_\bx\Sph}$:
$$
  \frac{1}{\kappa}\bigl(\nabla_\Sph\log p_\mathrm{kde}
  -\nabla_\Sph\log q_\mathrm{kde}\bigr)
  =\Proj_{T_\bx\Sph}\!\bigl(
   \E_{p,k_\kappa}[\by\mid\bx]
   -\E_{q,k_\kappa}[\by\mid\bx]\bigr)
  =\mathbf{V}_{p,q}^\Sph(\bx).
$$
\end{proof}

\subsubsection{Spherical Logarithmic Kernel}

\begin{proposition}
\label{prop:spherical-log}
Let $c > 0$ and $0 < \alpha < \frac{1}{2+c}$. The spherical logarithmic kernel defined by
\begin{equation}
  k_{c, \alpha}(\bx, \by) := -\log\bigl(\alpha(1 - \bx^\top\by + c)\bigr)
\end{equation}
satisfies all spherical assumptions \textbf{K1}$_\Sph$--\textbf{K4}$_\Sph$.
\end{proposition}

\begin{proof}
Let $z := \bx^\top\by$. Since $\bx, \by \in \Sph^{d-1}$, we have $z \in [-1, 1]$. Let the argument of the logarithm be denoted as $g(z) := \alpha(1 - z + c)$.

\textbf{K3}$_\Sph$ (Strict positivity):\;
We analyze the bounds of $g(z)$. Since $z \in [-1, 1]$, the term $1 - z + c$ achieves its minimum at $z=1$ (yielding $c$) and its maximum at $z=-1$ (yielding $2+c$). Therefore,
\begin{equation}
    0 < \alpha \cdot c \le g(z) \le \alpha(2+c).
\end{equation}
By the given condition $\alpha < \frac{1}{2+c}$, we strictly have $g(z) < 1$. Thus, $0 < g(z) < 1$ for all $\bx, \by \in \Sph^{d-1}$. Consequently, $k_{c,\alpha}(\bx,\by) = -\log(g(z)) > 0$ everywhere.

\textbf{K4}$_\Sph$ (Continuous differentiability):\;
Since $g(z) \ge \alpha \cdot c > 0$, the argument of the logarithm is strictly bounded away from zero. The functions $z = \bx^\top\by$ and $t \mapsto -\log(t)$ (for $t>0$) are smooth ($C^\infty$). Therefore, their composition $k_{c,\alpha} \in C^\infty(\Sph^{d-1} \times \Sph^{d-1})$, which trivially implies $C^1$.

\textbf{K2}$_\Sph$ (Uniform gradient bound):\;
The ambient gradient of the kernel with respect to $\bx$ is:
\begin{equation}
    \nabla_\bx k_{c,\alpha}(\bx, \by) 
    = \frac{d}{dz}\bigl[-\log(\alpha(1-z+c))\bigr] \nabla_\bx(\bx^\top\by) 
    = \frac{1}{1-z+c} \by.
\end{equation}
The Riemannian gradient is the projection onto the tangent space $T_\bx\Sph^{d-1}$:
\begin{equation}
    \nabla_{\Sph,\bx}k_{c,\alpha}(\bx,\by) 
    = \Proj_{T_\bx\Sph}(\nabla_\bx k_{c,\alpha}) 
    = \frac{1}{1-z+c} \bigl(\by - (\bx^\top\by)\bx\bigr).
\end{equation}
We compute its norm. Since $\bx$ and $\by$ are unit vectors, $\|\by - z\bx\|^2 = \|\by\|^2 - 2z(\bx^\top\by) + z^2\|\bx\|^2 = 1 - z^2$. Thus:
\begin{equation}
    \|\nabla_{\Sph,\bx}k_{c,\alpha}(\bx,\by)\| = \frac{\sqrt{1-z^2}}{1-z+c}.
\end{equation}
Since $\sqrt{1-z^2} \le 1$ and $1-z+c \ge c > 0$, we have the uniform bound $\|\nabla_{\Sph,\bx}k_{c,\alpha}\| \le \frac{1}{c} < \infty$.

\textbf{K1}$_\Sph$ (Characteristic):\;
We expand the kernel into a Taylor series:
\begin{align}
    k_{c,\alpha}(\bx,\by) 
    &= -\log\alpha - \log\bigl((1+c) - z\bigr) \\
    &= -\log\alpha - \log(1+c) - \log\left(1 - \frac{z}{1+c}\right).
\end{align}
Using the Maclaurin series for $-\log(1-x) = \sum_{n=1}^\infty \frac{x^n}{n}$ (which converges absolutely since $\left|\frac{z}{1+c}\right| \le \frac{1}{1+c} < 1$), we obtain:
\begin{equation}
    k_{c,\alpha}(\bx,\by) = \underbrace{-\log\bigl(\alpha(1+c)\bigr)}_{a_0} + \sum_{n=1}^\infty \underbrace{\frac{1}{n(1+c)^n}}_{a_n} (\bx^\top\by)^n.
\end{equation}
From the user condition $\alpha < \frac{1}{2+c} < \frac{1}{1+c}$, we have $\alpha(1+c) < 1$, which strictly implies $a_0 > 0$. For all $n \ge 1$, since $1+c > 1$, we clearly have $a_n > 0$.
Because the kernel $k_{c,\alpha}$ can be expressed as a power series $f(\bx^\top\by) = \sum_{n=0}^\infty a_n (\bx^\top\by)^n$ where $a_n > 0$ for \emph{all} $n \ge 0$, Schoenberg's theorem guarantees that it is strictly positive definite on $\Sph^{d-1}$ for any dimension $d \ge 2$. Thus, it is a universal (and therefore characteristic) kernel.
\end{proof}

\begin{remark}[Spherical Score for Logarithmic Kernel]
\label{rem:spherical-log-score}
Following the same logic as Proposition~\ref{prop:sphere-equiv}, the spherical KDE score for this logarithmic kernel can be explicitly derived. Define the pairwise weight function $W_{c}(\bx, \by) := \frac{1}{1 - \bx^\top\by + c}$. The ambient gradient of $\mu_\mathrm{kde}$ is $\int W_c(\bx,\by)\by \mathrm{d}\mu(\by)$. Consequently, the Riemannian score is:
\begin{equation}
    \nabla_\Sph\log\mu_\mathrm{kde}(\bx) = \Proj_{T_\bx\Sph}\left( \frac{\int W_c(\bx,\by)\by \mathrm{d}\mu(\by)}{\int k_{c,\alpha}(\bx,\by) \mathrm{d}\mu(\by)} \right).
\end{equation}
Unlike the vMF kernel, the weighting factor $W_c(\bx,\by)$ here is polynomial (specifically, inversely proportional to distance squared, analogous to the Euclidean IMQ kernel), which typically produces heavier tails and better global mode coverage.
\end{remark}

\end{document}